\documentclass{article}

\PassOptionsToPackage{numbers,compress}{natbib}
\PassOptionsToPackage{dvipsnames}{xcolor}

\usepackage[preprint]{neurips_2020}

\usepackage[utf8]{inputenc} 
\usepackage[T1]{fontenc}    
\usepackage{hyperref}       
\usepackage{url}            
\usepackage{booktabs}       
\usepackage{amsfonts}       
\usepackage{nicefrac}       
\usepackage{microtype}      

\usepackage{bm}
\usepackage{subfig}
\usepackage{graphicx}
\usepackage{booktabs}
\usepackage{amssymb,amsthm,amsmath}
\usepackage{multirow}
\usepackage{makecell}
\usepackage{textcomp}
\usepackage{array}
\usepackage{setspace}

\usepackage[linesnumbered,ruled,vlined]{algorithm2e}
\SetKwRepeat{Do}{do}{while}
\DontPrintSemicolon

\SetCommentSty{commentsty}

\usepackage{mathtools}

\DeclarePairedDelimiterX{\infdivx}[2]{(}{)}{%
  #1\;\delimsize\|\;#2%
}

\newtheorem{theorem}{Theorem}[section]

\theoremstyle{remark}
\newtheorem*{remark}{Remark}

\usepackage[dvipsnames]{xcolor}
\colorlet{CPPlinkcolor}{violet!85!black}
\colorlet{CPPcitecolor}{YellowOrange!85!black}
\colorlet{CPPurlcolor}{Aquamarine!85!black}
\hypersetup{
	linkcolor	= CPPlinkcolor,
	citecolor	= CPPcitecolor,
	urlcolor	= CPPurlcolor,
	colorlinks	= true,
	breaklinks	= true
}

\DeclareMathOperator{\GC}{GC}
\DeclareMathOperator{\MP}{\widehat{MP}}

\title{Deep Graph Contrastive Representation Learning}

\author{%
Yanqiao Zhu\textsuperscript{1,2}\thanks{The first two authors contributed equally to this work.} \quad Yichen Xu\textsuperscript{3}\footnotemark[1]\textsuperscript{~~,}\thanks{This work is done during his internship at CRIPAC, CASIA.} \quad Feng Yu\textsuperscript{1,2}\quad Qiang Liu\textsuperscript{4,5}\quad Shu Wu\textsuperscript{1,2}\quad Liang Wang\textsuperscript{1,2} \\
\textsuperscript{1} Center for Research on Intelligent Perception and Computing\\ Institute of Automation, Chinese Academy of Sciences\\
\textsuperscript{2} School of Artificial Intelligence, University of Chinese Academy of Sciences\\
\textsuperscript{3} School of Computer Science, Beijing University of Posts and Telecommunications\\
\textsuperscript{4} RealAI \quad \textsuperscript{5} Tsinghua University\\
\small\texttt{yanqiao.zhu@cripac.ia.ac.cn, linyxus@bupt.edu.cn}\\
\small\texttt{\{feng.yu, shu.wu, wangliang\}@nlpr.ia.ac.cn, qiang.liu@realai.ai} \\
}

\begin{document}

\maketitle

\begin{abstract}
Graph representation learning nowadays becomes fundamental in analyzing graph-structured data.
Inspired by recent success of contrastive methods, in this paper, we propose a novel framework for unsupervised graph representation learning by leveraging a contrastive objective at the node level. Specifically, we generate two graph views by corruption and learn node representations by maximizing the agreement of node representations in these two views. To provide diverse node contexts for the contrastive objective, we propose a hybrid scheme for generating graph views on both structure and attribute levels. Besides, we provide theoretical justification behind our motivation from two perspectives, mutual information and the classical triplet loss.
We perform empirical experiments on both transductive and inductive learning tasks using a variety of real-world datasets. Experimental experiments demonstrate that despite its simplicity, our proposed method consistently outperforms existing state-of-the-art methods by large margins. 
Moreover, our unsupervised method even surpasses its supervised counterparts on transductive tasks, demonstrating its great potential in real-world applications.

\end{abstract}

\section{Introduction}

Over the past few years, graph representation learning has emerged as a powerful strategy for analyzing graph-structured data. Graph representation learning aims to learn an encoding function that transforms nodes to low-dimensional dense embeddings that preserve graph attributive and structural features.
Traditional unsupervised graph representation learning approaches, such as DeepWalk \cite{Perozzi:2014ib} and node2vec \cite{Grover:2016ex}, follow a \emph{contrastive} framework originated in the skip-gram model \cite{Mikolov:2013uz}.
Specifically, they first sample short random walks and then enforce neighboring nodes on the same walk to share similar embeddings by contrasting them with other nodes. However, DeepWalk-based methods can be seen as reconstructing the graph proximity matrix, such as high-order adjacent matrix \cite{Qiu:2018ez}, which excessively emphasize proximity information defined on the network structure \cite{Ribeiro:2017ji}.

Recently, graph representation learning using Graph Neural Networks (GNN) has received considerable attention. Along with its prosperous development, however, there is an increasing concern over the label availability when training the model.
Nevertheless, existing GNN models are mostly established in a supervised manner \cite{Kipf:2016tc,Velickovic:2018we,Hu:2019vq}, which require abundant labeled nodes for training. Albeit with some attempts connecting previous unsupervised objectives (i.e., matrix reconstruction) to GNN models \cite{Kipf:2016ul,Hamilton:2017tp}, these methods still heavily rely on the preset graph proximity matrix.

Instead of optimizing the reconstruction objective, visual representation learning leads to revitalization of the classical information maximization (InfoMax) principle \cite{Linsker:1988ho}. A series of contrastive learning methods have been proposed so far \cite{Wu:2018kw,Tian:2019vw,He:2020tu,Bachman:2019wp,Ye:2019we,Chen:2020wj}, which seek to maximize the Mutual Information (MI) between the input (i.e., images) and its representations (i.e., image embeddings) by contrasting positive pairs with negative-sampled counterparts.
Inspired by previous success of the Deep InfoMax (DIM) method \cite{Bachman:2019wp} in visual representation learning, Deep Graph InfoMax (DGI) \cite{Velickovic:2019tu} proposes an alternative objective based on MI maximization in the graph domain. DGI firstly employs GNN to learn node embeddings and obtains a global summary embedding (i.e., the graph embedding), via a readout function. The objective used in DGI is then to maximize the MI between node embeddings and the graph embedding by discriminating nodes in the original graph from nodes in a corrupted graph.

However, we argue that the local-global MI maximization framework in DGI is still in its infancy. Its objective is proved to be equivalent to maximizing the MI between input node features and high-level node embeddings under some conditions. 
Specifically, to implement the InfoMax objective, DGI requires an injective readout function to produce the global graph embedding, where the injective property is too restrictive to fulfill.
For the mean-pooling readout function employed in DGI, it is not guaranteed that the graph embedding can distill useful information from nodes, as it is insufficient to preserve distinctive features from node-level embeddings.
Moreover, DGI proposes to use feature  shuffling to generate corrupted views of graphs.
Nevertheless, this scheme considers corrupting node features at a coarse-grained level when generating negative node samples. When the feature matrix is sparse, performing feature shuffling only is insufficient to generate different neighborhoods (i.e., contexts) for nodes in the corrupted graph, leading to difficulty in learning of the contrastive objective.


In this paper, we introduce a simple yet powerful contrastive framework for unsupervised graph representation learning (Figure \ref{fig:model}), which we refer to as deep \underline{GRA}ph \underline{C}ontrastive r\underline{E}presentation learning (GRACE)\footnote{Code is made publicly available at \texttt{\url{https://github.com/CRIPAC-DIG/GRACE}}.}, motivated by a traditional self-organizing network \cite{Becker:1992ts} and its recent renaissance in visual representation learning \cite{Chen:2020wj}.
Rather than contrasting node-level embeddings to global ones, we primarily focus on contrasting embeddings at the node level and our work makes no assumptions on injective readout functions for generating the graph embedding.
In GRACE, we first generate two correlated \emph{graph views} by randomly performing \emph{corruption}. Then, we train the model using a contrastive loss to maximize the agreement between node embeddings in these two views.
Unlike visual data, where abundant image transformation techniques are available, how to perform corruption to generate views for graphs is still an open problem.
In our work, we jointly consider corruption at both topology and node attribute levels, namely removing edges and masking features, to provide diverse contexts for nodes in different views, so as to boost optimization of the contrastive objective.
Last, we provide theoretical analysis that reveals the connections from our contrastive objective to mutual information and the classical triplet loss.

\begin{figure}[t]
	\centering
	\includegraphics[width=0.9\textwidth]{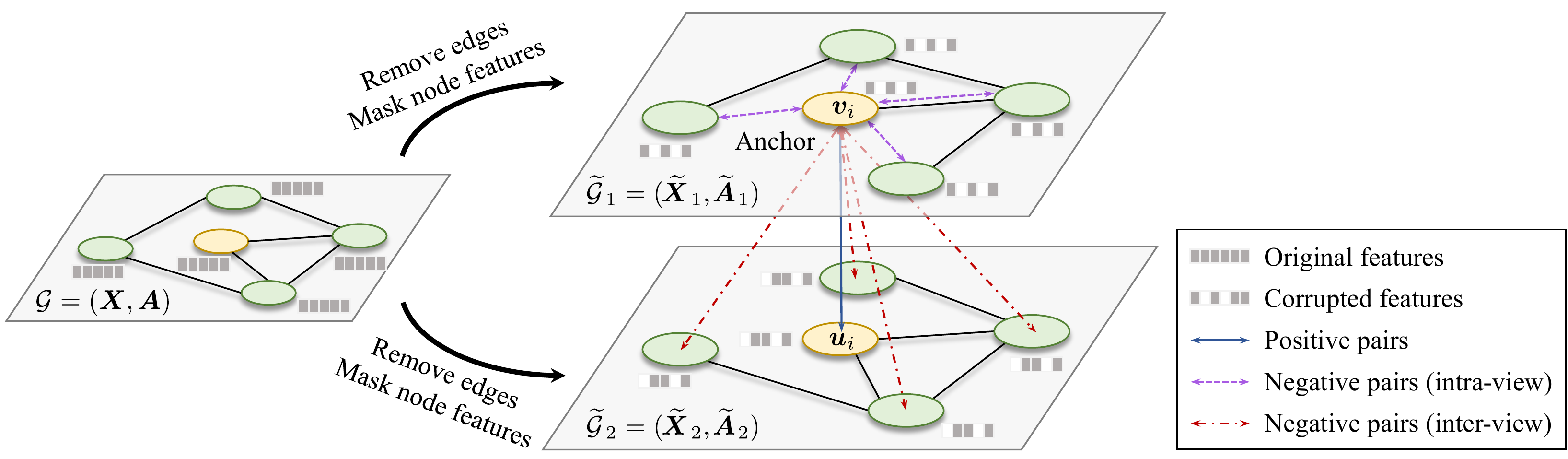}
	\caption{Our proposed deep GRAph Contrastive rEpresentation learning (GRACE) model.}
	\label{fig:model}
\end{figure}

Our contribution is summarized as follows.
Firstly, we propose a general contrastive framework for unsupervised graph representation learning. The proposed GRACE framework simplifies previous work and works by maximizing the agreement of node embeddings between two graph views.
Secondly, we propose two specific schemes, removing edges and masking features, to generate views of graphs.
Finally, we conduct comprehensive empirical studies using six popular public benchmark datasets on both transductive and inductive node classification under the commonly-used linear evaluation protocol. GRACE consistently outperforms existing methods
and our unsupervised method even surpasses its supervised counterparts on transductive tasks, demonstrating its great potential in real-world applications.

\section{Related Work}

\textbf{Contrastive learning of visual representations.\quad}
Being popular in self-supervised visual representation learning, contrastive methods aim to learn discriminative representations by contrasting positive and negative samples. For visual data, negative samples can be generated using image augmentation techniques such as cropping, rotation \cite{Gidaris:2018wr}, color distortion \cite{Larsson:2017vt}, etc.
Existing work \cite{Wu:2018kw,Tian:2019vw,He:2020tu} employs a memory bank for storing negative samples. Other work \cite{Bachman:2019wp,Ye:2019we,Chen:2020wj} explores in-batch negative samples. For an image patch as the anchor, these methods usually find a global summary vector \cite{Hjelm:2019uk,Bachman:2019wp} or patches in neighboring views \cite{vandenOord:2018ut,Henaff:2019ta} as the positive sample, and contrast them with negative-sampled counterparts, such as patches of other images within the same batch \cite{Hjelm:2019uk}.

Theoretical analysis sheds light on the reasons behind their success \cite{Poole:2019vk}. Objectives used in these methods can be seen as maximizing the lower bounds of MI between input features and their representations \cite{Linsker:1988ho}. However, recent work \cite{Tschannen:2020uo} reveals that downstream performance in evaluating the quality of representations may strongly depend on the bias that is encoded not only in the convolutional architectures but also in the specific estimator of the InfoMax objective.

\textbf{Graph representation learning.\quad}
Many traditional methods on unsupervised graph representation learning employ the contrastive paradigm as well \cite{Perozzi:2014ib,Grover:2016ex,Kipf:2016ul,Hamilton:2017wa}.
Prior work on unsupervised graph representation learning focuses on local contrastive patterns, which forces neighboring nodes to have similar embeddings. Positive samples under this circumstance are nodes appearing in the same random walk \cite{Perozzi:2014ib,Grover:2016ex}.
For example, the pioneering work DeepWalk \cite{Perozzi:2014ib} models probabilities of node co-occurrence pairs using noise-contrastive estimation \cite{Gutmann:2012eq}. These random-walk-based methods are proved to be equivalent to factorizing some forms of graph proximity (e.g., transformation of the adjacent matrix) \cite{Qiu:2018ez}, which overly emphasize on the structural information encoded in these graph proximities and also face severe scaling problem with large-scale datasets.
Also, these methods are known to be error-prone with inappropriate hyperparameter tuning \cite{Perozzi:2014ib,Grover:2016ex}.

Recent work on graph neural networks (GNN) employs more powerful graph convolutional encoders over conventional methods. Among them, considerable literature has grown up around the theme of supervised GNN \cite{Kipf:2016tc,Velickovic:2018we,Hu:2019vq,Wu:2019vz}, which requires labeled datasets that may not be accessible in real-world applications.
Along the other line of development, unsupervised GNNs receive little attention. Representative methods include GraphSAGE \cite{Hamilton:2017tp}, which incorporates DeepWalk-like objectives as well. Recent work DGI \cite{Velickovic:2019tu} marries the power of GNN and contrastive learning, which focuses on maximizing MI between global graph embeddings and local node embeddings. However, it is hard to fulfill the injective requirement of the graph readout function such that the graph embedding may be deteriorated. In contrast to DGI, our work does not rely on an explicit graph embedding. Instead, we focus on maximizing the agreement of node embeddings across two corrupted views of the graph.


\makeatletter
\def\mathcenterto#1#2{\mathclap{\phantom{#1}\mathclap{#2}}\phantom{#1}}
\let\old@widetilde\widetilde
\def\widetildeto#1#2{\mathcenterto{#2}{\old@widetilde{\mathcenterto{#1}{#2\,}}}}
\let\old@widehat\widehat
\def\widehatto#1#2{\mathcenterto{#2}{\old@widehat{\mathcenterto{#1}{#2\,}}}}
\makeatother
\def\widetilde{\widetildeto{X}}

\section{Deep Graph Contrastive Representation Learning}

In this section, we present our proposed GRACE framework in detail, starting with the overall framework of contrastive objectives, followed by specific graph view generation methods. At the end of this section, we provide theoretical justification behind our framework from two perspectives, i.e., connection to the InfoMax principle and the classical triplet loss.

\subsection{Preliminaries}
In unsupervised graph representation learning, let \(\mathcal{G} = (\mathcal{V}, \mathcal{E})\) denote a graph, where \(\mathcal{V} = \{ v_1, v_2, \cdots, v_N\}\), \(\mathcal{E} \subseteq \mathcal V \times \mathcal V\) represent the node set and the edge set respectively. We denote the feature matrix and the adjacency matrix as \(\bm{X} \in \mathbb{R}^{N \times F}\) and \(\bm{A} \in \{0,1\}^{N \times N}\), where \(\bm{x}_i \in \mathbb{R}^{F}\) is the feature of \(v_i\), and \(\bm{A}_{ij} = 1\) iff \((v_i, v_j) \in \mathcal{E}\).
There is no given class information of nodes in \(\mathcal{G}\) during training.
Our objective is to learn a GNN encoder \(f(\bm{X}, \bm{A}) \in \mathbb{R}^{N \times F^\prime}\) receiving the graph features and structure as input, that produces node embeddings in low dimensionality, i.e., \(F^\prime \ll F\). We denote \(\bm{H} = f(\bm{X}, \bm{A})\) as the learned representations of nodes, where \(\bm{h}_i\) is the embedding of node \(v_i\). These representations can be used in downstream tasks, such as node classification.


\subsection{Contrastive Learning of Node Representations}

\subsubsection{The Contrastive Learning Framework} 

Contrary to previous work that learns representations by utilizing local-global relationships, in GRACE, we learn embeddings by directly maximizing node-level agreement between embeddings.
To be specific, we first generate two graph views by randomly corrupting the original graph. Then, we employ a contrastive objective that enforces the encoded embeddings of each node in the two different views agree with each other and can be distinguished from embeddings of other nodes.

In our GRACE model, at each iteration, we generate two graph views, denoted as \(\widetilde{G}_1\) and \(\widetilde G_2\), and denote node embeddings in the two generated views as \(\bm{U} = f(\widetilde{\bm{X}}_1, \widetilde{\bm A}_1)\) and \(\bm{V} = f(\widetilde{\bm{X}}_2, \widetilde{\bm A}_2)\), where \(\widetilde{\bm{X}}_{*}\) and \(\widetilde{\bm{A}}_{*}\) are the feature matrices and adjacent matrices of the views.
Details on the generation of graph views will be discussed later in Section \ref{sec:generate-views}.

 
Then, we employ a contrastive objective (i.e., a discriminator) that distinguishes the embeddings of the same node in these two different views from other node embeddings. For any node \(v_i\), its embedding generated in one view, \(\bm{u}_i\), is treated as the anchor, the embedding of it generated in the other view, \(\bm{v}_i\), forms the positive sample, and embeddings of nodes other than \(v_i\) in the two views are naturally regarded as negative samples. Formally, we define the critic \(\theta(\bm u, \bm v) = s(g(\bm u), g(\bm v))\), where \(s\) is the cosine similarity and \(g\) is a non-linear projection to enhance the expression power of the critic \cite{Chen:2020wj,Tschannen:2020uo}. The projection \(g\) is implemented with a two-layer multilayer perceptron (MLP). We define the pairwise objective for each positive pair \((\bm{u}_i, \bm{v}_i)\) as
\begin{equation}
	\ell(\bm{u}_i, \bm{v}_i) = \log \frac {e^{\theta\left(\bm{u}_i, \bm{v}_{i} \right) / \tau}} {\underbrace{e^{\theta\left(\bm{u}_i, \bm{v}_{i} \right) / \tau}}_{\text{the positive pair}} + \underbrace{\sum_{k=1}^{N} \mathds{1}_{[k \neq i]} e^{\theta\left(\bm{u}_i, \bm{v}_{k} \right) / \tau}}_{\text{inter-view negative pairs}} + \underbrace{\sum_{k=1}^{N} \mathds{1}_{[k \neq i]} e^{\theta\left(\bm{u}_i, \bm{u}_k \right) / \tau}}_{\text{intra-view negative pairs}}},
\end{equation}
where \(\mathds{1}_{[k \neq i]} \in \{ 0, 1 \}\) is an indication function that equals to \(1\) iff \(k \neq i\), and \(\tau\) is a temperature parameter. Please note that, in our work, we do not sample negative nodes \emph{explicitly}. Instead, given a positive pair, we naturally define negative samples as all other nodes in the two views. Therefore, negative samples come from two sources, inter-view or intra-view nodes, corresponding to the second and the third term in the denominator, respectively. Since two views are symmetric, the loss for another view is defined similarly for \(\ell(\bm{v}_i, \bm u_i)\). 
The overall objective to be maximized is then defined as the average over all positive pairs, formally given by
\begin{equation}
	\mathcal{J} = \frac{1}{2N} \sum_{i = 1}^{N} \left[\ell(\bm{u}_i, \bm{v}_i) + \ell(\bm{v}_i, \bm{u}_i)\right].
	\label{eq:overall-loss}
\end{equation}

To sum up, at each training epoch, GRACE first generates two graph views \(\widetilde{\mathcal G}_1\) and \(\widetilde{\mathcal G}_2\) of graph \(\mathcal G\). Then, we obtain node representations \(\bm{U}\) and \(\bm{V}\) of \(\widetilde{\mathcal{G}}_1\) and \(\widetilde{\mathcal{G}}_2\) using a GNN encoder \(f\). Finally, the parameters of \(f\) and \(g\) is updated by maximizing the objective in Eq. (\ref{eq:overall-loss}). The learning algorithm is summarized in Algorithm \ref{algo:GRACE-training}.

\begin{algorithm}[ht]
	\DontPrintSemicolon\SetNoFillComment
	\caption{GRACE training algorithm}
	\label{algo:GRACE-training}
	\For {\(epoch \gets 1, 2, \cdots\)} {
		Generate two graph views \(\widetilde{\mathcal{G}}_1\) and \(\widetilde{\mathcal{G}}_2\) by performing corruption on \(\mathcal{G}\)\;
		Obtain node embeddings \(\bm{U}\) of \(\widetilde{\mathcal{G}}_1\) using the encoder \(f\)\;
		Obtain node embeddings \(\bm{V}\) of \(\widetilde{\mathcal{G}}_2\) using the encoder \(f\)\;
		Compute the contrastive objective \(\mathcal{J}\) with Eq. (\ref{eq:overall-loss})\;
		Update parameters by applying stochastic gradient ascent to maximize \(\mathcal{J}\)\;
	}
\end{algorithm}

\subsubsection{Graph View Generation}
\label{sec:generate-views}
Generating views is a key component of contrastive learning methods. In the graph domain, different views of a graph provide different contexts for each node. Considering contrastive approaches that rely on contrasting between node embeddings in different views, we propose to corrupt the original graph at both structure and attribute levels, which constructs diverse node contexts for the model to contrast with.
In GRACE, we design two methods for graph corruption, removing edges for topology and masking features for node attributes.
How to perform graph corruption is still an open problem \cite{Velickovic:2019tu}. It is flexible to adopt other alternative mechanisms of corruption methods in our framework.



\textbf{Removing edges (RE). \quad}
We randomly remove a portion of edges in the original graph. Formally, since we only remove existing edges, we first sample a random masking matrix \(\widetilde{\bm{R}} \in \{ 0, 1 \}^{N \times N}\), whose entry is drawn from a Bernoulli distribution \(\widetilde{\bm{R}}_{ij} \sim \mathcal{B}(1 - p_r)\) if \(\bm{A}_{ij} = 1\) for the original graph and \(\widetilde{\bm{R}}_{ij} = 0\) otherwise. Here \(p_r\) is the probability of each edge being removed. The resulting adjacency matrix can be computed as
\begin{equation}
	\widetilde{\bm A} = \bm{A} \circ \widetilde{\bm{R}},
\end{equation}
where \((\bm{x} \circ \bm{y})_i = x_i y_i\) is Hadamard product.

\textbf{Masking node features (MF). \quad}
Apart from removing edges, we randomly mask a fraction of dimensions with zeros in node features. Formally, we first sample a random vector \(\widetilde{\bm{m}} \in \{ 0, 1 \}^F\) where each dimension of it independently is drawn from a Bernoulli distribution with probability \(1 - p_m\), i.e., \(\widetilde{m}_i \sim \mathcal{B}(1 - p_m), \forall i\). Then, the generated node features \(\widetilde {\bm X}\) is computed by
\begin{equation}
	\widetilde {\bm{X}} = [ \bm x_1 \circ \widetilde{\bm{m}}; \bm x_2 \circ \widetilde{\bm{m}}; \cdots; \bm{x}_N \circ \widetilde{\bm{m}} ]^\top.
\end{equation}
Here \([\cdot ; \cdot ]\) is the concatenation operator.

Please kindly note that although our proposed RE and MF schemes are technically similar to Dropout \cite{Srivastava:2014cg} and DropEdge \cite{Rong:2020vx}, our GRACE model and the two referred methods are proposed for fundamentally different purposes. Dropout is a general technique that randomly masks neurons during training to prevent over-fitting of large-scale models. In the graph domain, DropEdge is proposed to prevent over-fitting and alleviate over-smoothing \emph{when the GNN architecture is too deep}. However, our GRACE framework randomly applies RE and MF to produce different graph views for contrastive learning at both graph topology and node feature levels. Moreover, the employed GNN encoder in GRACE is a rather shallow model, usually consisting of only two or three layers.


In our implementation, we jointly leverage these two methods to generate graph views.
The generation of \(\widetilde{\mathcal G}_1\) and \(\widetilde{\mathcal G}_2\) are controlled by two hyperparameters \(p_r\) and \(p_m\). 
To provide different contexts in the two views, the generation process of the two views uses two different sets of hyperparameters \(p_{r,1}, p_{m,1}\) and \(p_{r,2}, p_{m,2}\).
Experiments demonstrate that our model is not sensitive to the choice of \(p_r\) and \(p_m\) under mild conditions such that the original graph is not overly corrupted, e.g., \(p_r \le 0.8\) and \(p_m \le 0.8\). We refer readers to the sensitivity analysis presented in Appendix \ref{appendix:sensitivity} for empirical results.

\subsection{Theoretical Justification}

In this section, we provide theoretical justification behind our model from two perspectives, i.e., the mutual information maximization and the triplet loss. Detailed proofs can be found in Appendix \ref{appendix:proofs}.

\textbf{Connections to the mutual information.\quad}
Firstly, we reveal the connection between our loss and mutual information maximization between node features and the embeddings in the two views, which has been widely applied in the representation learning literature \cite{Tian:2019vw,Bachman:2019wp,Poole:2019vk,Tschannen:2020uo}. MI quantifies the amount of information obtained about one random variable by observing the other random variable.


\begin{theorem}
\label{thm:objective-InfoMax}
Let \(\mathbf{X}_i = \{ \bm{x}_k \}_{k \in \mathcal{N}(i)}\) be the neighborhood of node \(v_i\) that collectively maps to its output embedding, where \(\mathcal{N}(i)\) denotes the set of neighbors of node \(v_i\) specified by GNN architectures, and \(\mathbf{X}\) be the corresponding random variable with a uniform distribution \(p(\mathbf{X}_i) = \frac{1}{N}\). Given two random variables \(\mathbf{U, V} \in \mathbb{R}^{F^\prime}\) being the embedding in the two views, with their joint distribution denoted as \(p(\mathbf{U}, \mathbf{V})\), our objective \(\mathcal{J}\) is a lower bound of MI between encoder input \(\mathbf{X}\) and node representations in two graph views \(\mathbf{U, V}\). Formally,
\begin{equation}
	\mathcal{J} \leq I(\mathbf{X}; \mathbf{U}, \mathbf{V}).
\end{equation}
\end{theorem}
\begin{proof}[Proof sketch]
We first observe that our objective \(\mathcal{J}\) is a lower bound of the InfoNCE objective \cite{vandenOord:2018ut}, which is defined as \(I_\text{NCE}(\mathbf U; \mathbf V) \triangleq \mathbb{E}_{\prod_i {p(\bm u_i, \bm v_i)}} \left[ \frac{1}{N} \sum_{i=1}^N \log \frac{e^{\theta(\bm{u}_i, \bm{v}_i)}}{\frac{1}{N}\sum_{j = 1}^{N} e^{\theta(\bm{u}_i, \bm{v}_j)}} \right] \) \cite{Poole:2019vk}. According to \cite{vandenOord:2018ut}, the InfoNCE estimator is a lower bound of the true MI. Therefore, the theorem directly follows from the application of data processing inequality, which states that \(I(\mathbf U; \mathbf V) \leq I(\mathbf X; \mathbf U, \mathbf V)\).
\end{proof}
\begin{remark}
From Theorem \ref{thm:objective-InfoMax}, it reveals that maximizing \(\mathcal{J}\) is equivalent to maximizing a lower bound of the mutual information \(I(\mathbf X; \mathbf U, \mathbf V)\) between input node features and learned node representations.
Counterintuitively, recent work further provides empirical evidence that optimizing a stricter bound of MI may not lead to better downstream performance on visual representation learning \cite{Tschannen:2020uo}, which highlights the importance of the encoder design. In Appendix \ref{appendix:InfoNCE-objective}, we also compare our objective with the InfoNCE loss, which is a stricter estimator of MI, to further demonstrate the superiority of the GRACE model.
\end{remark}

\textbf{Connections to the triplet loss.\quad}
Alternatively, we may view the optimization problem in Eq. (\ref{eq:overall-loss}) as a classical triplet loss, commonly used in deep metric learning.
\begin{theorem}
\label{thm:objective-triplet-loss}
When the projection function \(g\) is the identity function and we measure embedding similarity by simply taking inner product, i.e., \(s(\bm{u}, \bm{v}) = \bm{u}^\top \bm{v}\), and further assuming that positive pairs are far more aligned than negative pairs, minimizing the pairwise objective \(\ell(\bm u_i, \bm v_i)\) coincides with maximizing the triplet loss, as given in the sequel
\begin{equation}
	- \ell(\bm u_i, \bm v_i) \propto 4N \tau + \sum_{j=1}^N \mathds 1_{[j \neq i]} \left[ \left(\| {\bm u_i} - {\bm v_i} \|^2 - \| {\bm u_i} - {\bm v_j} \|^2\right) + \left(\| {\bm u_i} - {\bm v_i} \|^2 - \| {\bm u_i} - {\bm u_j} \|^2\right) \right].
\end{equation}
\end{theorem}
\begin{remark}
Theorem \ref{thm:objective-triplet-loss} draws connection between the objective and the classical triplet loss.
In other words, we may regard the problem in Eq. (\ref{eq:overall-loss}) as learning graph convolutional encoders to encourage positive samples being further away from negative samples in the embedding space. Moreover, by viewing the objective from the metric learning perspective, we highlight the importance of appropriately choosing negative samples, which is often neglected in previous InfoMax-based methods. Last, the contrastive objective is cheap to optimize since we do not have to generate negative samples explicitly and all computation can be performed in parallel. In contrast, the triplet loss is known to be computationally expensive \cite{Schroff:2015wo}.
\end{remark}

\section{Experiments}

In this section, we empirically evaluate the quality of produced node embeddings on node classification using six public benchmark datasets.
We refer readers of interest to the supplementary material on details of experiments, including dataset configurations (Appendix \ref{appendix:dataset}), implementation and hyperparameters (Appendix \ref{appendix:implementation}), and additional experiments (Appendix \ref{appendix:experiments}).

\subsection{Datasets}

For comprehensive comparison, we use six widely-used datasets to study the performance of both transductive and inductive node classification. Specifically, we use three kinds of datasets:
(1) citation networks including Cora, Citeseer, Pubmed, and DBLP \cite{Sen:2008gm,Bojchevski:2018ua} for transductive node classification,
(2) social networks from Reddit posts for inductive learning on large-scale graphs \cite{Hamilton:2017tp},
and (3) biological protein-protein interaction (PPI) networks \cite{Zitnik:2017uz} for inductive node classification on multiple graphs.
Details of these datasets can be found in Appendix \ref{appendix:dataset}.

\subsection{Experimental Setup}

For every experiment, we follow the linear evaluation scheme as in \cite{Velickovic:2019tu}, where each model is firstly trained in an unsupervised manner. The resulting embeddings are used to train and test a simple \(\ell_2\)-regularized logistic regression classifier.
We train the model for twenty runs and report the averaged performance on each dataset. Moreover, we measure performance using micro-averaged F1-score on inductive tasks and accuracy on transductive tasks.
Please kindly note that for inductive learning tasks, tests are conducted on unseen or untrained nodes and graphs, while for transductive learning tasks, we use the features of all data, but the labels of the test set are masked during training.


\textbf{Transductive learning.\quad}
In transductive learning tasks, we employ a two-layer GCN \cite{Kipf:2016tc} as the encoder. Our encoder architecture is formally given by
\begin{align}
	\GC_i (\bm{X}, \bm{A}) & = \sigma \left( \hat{\bm{D}}^{-\frac{1}{2}} \hat {\bm{A}} \hat{\bm{D}}^{-\frac{1}{2}} \bm{X} \bm{W}_i \right), \\
	f(\bm X, \bm A) & = \GC_2 ( \GC_1 ( \bm{X}, \bm{A} ), \bm{A} ).
\end{align}
where \(\hat{\bm{A}} = \bm{A} + \bm{I}\) is the adjacency matrix with self-loops, \(\hat {\bm D} = \sum_i \hat{\bm{A}}_i\) is the degree matrix, \(\sigma(\cdot)\) is a nonlinear activation function, e.g., \(\operatorname{ReLU}(\cdot) = \max(0, \cdot)\), and \(\bm{W}_i\) is a trainable weight matrix.

We consider the following two categories of representative algorithms as baselines, including (1) traditional methods DeepWalk \cite{Perozzi:2014ib} and node2vec \cite{Grover:2016ex}, and (2) deep learning methods GAE, VGAE \cite{Kipf:2016ul}, and DGI \cite{Velickovic:2019tu}. Furthermore, we report performance obtained using a logistic regression classifier on raw node features and DeepWalk with embeddings concatenated with input node features. For direct comparison with supervised counterparts, we also report the performance of two representative models SGC \cite{Wu:2019vz} and GCN \cite{Kipf:2016tc}, where they are trained in an end-to-end fashion.

\textbf{Inductive learning on large graphs.\quad}
Considering the large scale of the Reddit data, we closely follow \cite{Velickovic:2019tu} and employ a three-layer GraphSAGE-GCN \cite{Hamilton:2017tp} with residual connections \cite{He:2016ib} as the encoder, which is formulated as
\begin{align}
	\MP_i(\bm{X}, \bm{A}) & = \sigma( [\hat{\bm{D}}^{-1}\hat{\bm{A}}\bm{X} ; \bm{X}] \bm{W}_i ),\label{eq:SAGE-mean-pooling}\\
	f(\bm{X}, \bm{A}) & = \MP_3 ( \MP_2 ( \MP_1 ( \bm{X}, \bm{A} ) , \bm{A} ), \bm{A} ).
\end{align}
Here we use the mean-pooling propagation rule, as \(\hat{\bm{D}}^{-1}\) averages over node features. Due to the large scale of Reddit, it cannot fit into GPU memory entirely. Therefore, we apply the subsampling method proposed in \cite{Hamilton:2017tp}, where we first randomly select a minibatch of nodes, then a subgraph centered around each selected node is obtained by sampling node neighbors with replacement. To be specific, we sample 30, 25, 20 neighbors at the first-, second-, and third-hop respectively. For generating graph views under such sampling-based settings, both RE and MF can be adapted to sampled subgraphs effortlessly.

\textbf{Inductive learning on multiple graphs.\quad}
For inductive learning on multiple graphs PPI, we stack three mean-pooling layer with skip connections, similar to DGI \cite{Velickovic:2019tu}. The graph convolutional encoder can be formulated as
\begin{align}
	\bm H_1 & = \MP_1(\bm X, \bm A), \\
	\bm H_2 & = \MP_2(\bm X \bm W_\text{skip} + \bm H_1, \bm A), \\
	f(\bm X, \bm A) = \bm H_3 &= \MP_3(\bm X \bm W_\text{skip}' + \bm H_1 + \bm H_2, \bm A),
\end{align}
where \(\bm{W}_\text{skip}\) and \(\bm{W}_\text{skip}'\) are two projection matrices, and \(\MP_i\) is defined in Eq. (\ref{eq:SAGE-mean-pooling}).
Despite that the PPI dataset consists of multiple graphs, we only compute negative samples for one anchor node as other nodes within the same graph, due to efficiency considerations.

Baselines in both large graphs and multiple graphs settings are selected similarly to transductive tasks. We consider (1) traditional methods DeepWalk\footnote{DeepWalk is not applicable to the multi-graph experiments, since the embedding spaces produced by DeepWalk may be arbitrarily rotated with respect to different disjoint graphs \cite{Hamilton:2017tp}.} \cite{Perozzi:2014ib}, and (2) deep learning methods GraphSAGE \cite{Hamilton:2017tp} and DGI \cite{Velickovic:2019tu}. Additionally, we report the performance of using raw features and DeepWalk + features under the same settings as in transductive tasks. We further provide the performance of two representative supervised methods, including FastGCN \cite{Chen:2018wp} and GaAN-mean \cite{Zhang:2018vn} for reference. In the table, results of baselines are reported in accordance with performance in their original papers. For GraphSAGE, we reuse the unsupervised results for fair comparison.


\subsection{Results and Analysis}

\begin{table}[t]
	\centering
	\caption{Summary of performance on node classification in terms of accuracy in percentage (on transductive tasks) or micro-averaged F1 score (on inductive tasks) with standard deviation. Available data for each method during the training phase is shown in the second column, where \(\bm{X}, \bm{A}, \bm{Y}\) correspond to node features, the adjacency matrix, and labels respectively. The highest performance of unsupervised models is highlighted in boldface.}
	\label{tab:node-classification}
	\vspace{4pt}
	\begin{tabular}{cccccc}
	\multicolumn{6}{c}{(a) \textit{Transductive}} \\
	\toprule
	Method & Training Data & Cora & Citeseer & Pubmed & DBLP \\
	\midrule
	Raw features & \(\bm{X}\) & 64.8 & 64.6 & 84.8 & 71.6 \\
	node2vec & \(\bm{A}\) & 74.8 & 52.3 & 80.3 & 78.8 \\
	DeepWalk & \(\bm{A}\) & 75.7 & 50.5 & 80.5 & 75.9 \\
	DeepWalk + features & \(\bm{X},\bm{A}\) & 73.1 & 47.6 & 83.7 & 78.1 \\
	\midrule
	GAE   & \(\bm{X},\bm{A}\) & 76.9 & 60.6 & 82.9 & 81.2 \\
	VGAE  & \(\bm{X},\bm{A}\) & 78.9 & 61.2 & 83.0 & 81.7 \\
	DGI   & \(\bm{X},\bm{A}\) & 82.6{\footnotesize \textpm0.4} & 68.8{\footnotesize \textpm0.7} & 86.0{\footnotesize \textpm0.1} & 83.2{\footnotesize \textpm0.1} \\
	\textbf{GRACE} & \(\bm{X},\bm{A}\) & \textbf{83.3{\footnotesize \textpm0.4}} & \textbf{72.1{\footnotesize \textpm0.5}} & \textbf{86.7{\footnotesize \textpm0.1}} & \textbf{84.2{\footnotesize \textpm0.1}} \\
	\specialrule{0.5pt}{0.5pt}{1pt}
	\midrule
    SGC & \(\bm{X},\bm{A},\bm{Y}\) & 80.6 & 69.1 & 84.8 & 81.7 \\
	GCN   & \(\bm{X},\bm{A},\bm{Y}\) & 82.8 & 72.0 & 84.9 & 82.7 \\
	\bottomrule
	\end{tabular}
	\\
	\begin{tabular}{cccc}
	\multicolumn{4}{c}{} \\
	\multicolumn{4}{c}{(b) \textit{Inductive}} \\
	\toprule
	Method & Training Data & Reddit & PPI \\
	\midrule
	Raw features & \(\bm{X}\) & 58.5 & 42.2 \\
	DeepWalk & \(\bm{A}\) & 32.4 & --- \\
	DeepWalk + features & \(\bm{X}, \bm{A}\) & 69.1 & --- \\
	\midrule
	GraphSAGE-GCN  & \(\bm{X}, \bm{A}\) & 90.8 & 46.5 \\
	GraphSAGE-mean & \(\bm{X}, \bm{A}\) & 89.7 & 48.6 \\
	GraphSAGE-LSTM & \(\bm{X}, \bm{A}\) & 90.7 & 48.2 \\
	GraphSAGE-pool & \(\bm{X}, \bm{A}\) & 89.2 & 50.2 \\
	DGI   & \(\bm{X}, \bm{A}\) & 94.0{\footnotesize \textpm0.1} & 63.8{\footnotesize \textpm0.2} \\
	\textbf{GRACE} & \(\bm{X}, \bm{A}\) & \textbf{94.2{\footnotesize \textpm0.0}} & \textbf{66.2{\footnotesize \textpm0.1}} \\
	\specialrule{0.5pt}{0.5pt}{1pt}
	\midrule
	FastGCN & \(\bm{X},\bm{A},\bm{Y}\) & 93.7 & --- \\
	GaAN-mean  & \(\bm{X},\bm{A},\bm{Y}\) & 95.8{\footnotesize \textpm0.1} &  96.9{\footnotesize \textpm0.2} \\
	\bottomrule
	\end{tabular}
\end{table}

The empirical performance is summarized in Table \ref{tab:node-classification}. Overall, from the table, we can see that our proposed model shows strong performance across all six datasets.
GRACE consistently performs better than unsupervised baselines by considerable margins on both transductive and inductive tasks. The strong performance verifies the superiority of the proposed contrastive learning framework.
We particularly note that GRACE is competitive with models \emph{trained with label supervision} on all four transductive datasets and inductive dataset Reddit.

We make other observations as follows.
Firstly, GRACE achieves considerable improvement over another competitive contrastive learning method DGI on PPI. We believe that this is due to the extreme sparsity of node features (over 40\% nodes having all-zero features \cite{Hamilton:2017tp}), which emphasizes the importance of considering topological information when choosing negative samples. For datasets like PPI, extreme feature sparsity prevents DGI from discriminating samples in the original graph from the corrupted graph, generated via shuffling node features, since shuffling node features makes no effect for the contrastive objective.
Contrarily, the RE scheme used in GRACE does not rely on node features and acts as a remedy under such circumstances, which can explain the gain of GRACE on PPI compared with DGI.
Also, we note that there is still a huge gap between our method with supervised models. These supervised models benefit another merit from labels, which provide other auxiliary information for model learning.
Considering the sparse nature of real-world datasets, we perform another experiment to verify that our method is robust to sparse node features (Appendix \ref{appendix:robustness}). Results show that by randomly removing node features, our still outperforms existing baselines.

Secondly, the performance of traditional contrastive learning methods like DeepWalk is inferior to the naive classifier that only uses raw features on some datasets (Citeseer, Pubmed, and Reddit), which suggests that these methods may be ineffective in utilizing node features.
Unlike traditional work, we see that GCN-based methods, e.g., GraphSAGE and GAE, are capable of incorporating node features when learning embeddings. However, we note that on certain datasets (Pubmed), their performance is still worse than DeepWalk + feature, which we believe can be attributed to their naive method of selecting negative samples that simply chooses contrastive pairs based on edges.
This fact further demonstrates the important role of selecting negative samples in contrastive representation learning. The superior performance of GRACE compared to GAEs also once again verifies the effectiveness of our proposed GRACE framework that contrasts nodes across graph views.


Additionally, we perform sensitivity analysis on critical hyperparameters \(p_r\) and \(p_m\) (Appendix \ref{appendix:sensitivity}) as well as ablation studies on our hybrid scheme on generating graph views (Appendix \ref{appendix:ablation}). Results show that our method is stable to perturbation of these parameters and verify the necessity of corruption at both graph topology and node feature levels. We also compare the classical InfoNCE loss (Appendix \ref{appendix:InfoNCE-objective}), verifying the efficacy of our design choice. Details of these extra experiments can be found in the supplementary material.

\section{Conclusion}

In this paper, we have developed a novel graph contrastive representation learning framework based on maximizing the agreement at the node level. Our model learns representations by first generating graph views using two proposed schemes, removing edges and masking node features, and then applying a contrastive loss to maximize the agreement of node embeddings in these two views. Theoretical analysis reveals the connections from our contrastive objective to mutual information maximization and the classical triplet loss, which justifies our motivation. We have conducted comprehensive experiments using various real-world datasets under transductive and inductive settings. Experimental results demonstrate that our proposed method can consistently outperform existing state-of-the-art methods by large margins and even surpass supervised counterparts on transductive tasks.

\section*{Discussions on Broader Impact}

Our proposed self-supervised graph representation learning techniques help alleviate the label scarcity issue when deploying machine learning applications in real-world, which saves a lot of efforts on human annotating.
For example, our GRACE framework can be plugged into existing recommender systems and produces high-quality embeddings for users and commodities to resolve the cold start problem.
Moreover, from the empirical results, our work outperforms existing baselines on protein function prediction by significant margins, which demonstrate its great potential in drug discovery and treatment, given the COVID-19 crisis at this critical juncture.
Note that our work mainly serves as a plug-in for existing machine learning models, it does not bring new ethical concerns.
However, the GRACE model may still give biased outputs (e.g., gender bias, ethnicity bias), as the provided data itself may be strongly biased during the processes of the data collection, graph construction, etc.

\section*{Acknowledgements}

The authors would like to thank Tao Sun and Sirui Lu for insightful discussions.
This work is jointly supported by National Key Research and Development Program (2018YFB1402600, 2016YFB1001000) and National Natural Science Foundation of China (U19B2038, 61772528).

\clearpage
\appendix
\section{Dataset Details}
\label{appendix:dataset}

\paragraph{Transductive learning.}
We utilize four widely-used citation networks, Cora, Citeseer, Pubmed, and DBLP, for predicting article subject categories.
In these datasets, graphs are constructed from computer science article citation links. Specifically, nodes correspond to articles and undirected edges to citation links between articles. Furthermore, each node has a sparse bag-of-words feature and a corresponding label of article types.
The former three networks are provided by \cite{Sen:2008gm,Yang:2016ts} and the latter DBLP dataset is provided by \cite{Bojchevski:2018ua}.
On these citation networks, we randomly select 10\% of the nodes as the training set, 10\% nodes as the validation set, and leave the rest nodes as the test set.

\paragraph{Inductive learning on large graphs.}
We then predict community structures of a large-scale social network, collected from Reddit. The dataset, preprocessed by \cite{Hamilton:2017tp}, contains Reddit posts created in September 2014, where posts belong to different communities (subreddit). In the dataset, nodes correspond to posts, and edges connect posts if the same user has commented on both.
Node features are constructed from post title, content, and comments, using off-the-shelf GloVe word embeddings \cite{Pennington:2014kh}, along with other metrics such as post score and the number of comments. 
Following the inductive setting of \cite{Hamilton:2017tp,Velickovic:2019tu}, on the Reddit dataset, we choose posts in the first 20 days for training, including 151,708 nodes, and the remaining for testing (with 30\% data including 23,699 nodes for validation).

\paragraph{Inductive learning on multiple graphs.}
Last, we predict protein roles, in terms of their cellular functions from gene ontology, within the protein-protein interaction (PPI) networks \cite{Zitnik:2017uz} to evaluate the generalization ability of the proposed method across multiple graphs.
The PPI dataset contains multiple graphs, with each corresponding to a human tissue.
The graphs are constructed by \cite{Hamilton:2017tp}, where each node has multiple labels that is a subset of gene ontology sets (121 in total), and node features include positional gene sets, motif gene sets, and immunological signatures (50 in total).
Following \cite{Hamilton:2017tp}, we select twenty graphs consisting of 44,906 nodes as the training set, two graphs containing 6,514 nodes as the validation, and the rest four graphs containing 12,038 nodes as the test set.

The statistics of datasets are summarized in Table \ref{tab:dataset-statistics}; download links are included in Table \ref{tab:dataset-links}.
For transductive tasks, similar to \cite{Kipf:2016tc}, during the training phase, all node features are visible but node labels are masked. In the inductive setting, we closely follow \cite{Hamilton:2017tp}; during training, nodes for evaluation are completely invisible; evaluation is then conducted on unseen or untrained nodes and graphs.

\begin{table}[h]
	\centering
	\caption{Statistics of datasets used in experiments.}
	\label{tab:dataset-statistics}
	\begin{tabular}{cccccc}
	\toprule
	Dataset & Type & \#Nodes & \#Edges & \#Features & \#Classes \\
	\midrule
	Cora & Transductive & 2,708 & 5,429 & 1,433 & 7 \\
	Citeseer & Transductive & 3,327 & 4,732 & 3,703 & 6 \\
	Pubmed & Transductive & 19,717 & 44,338 & 500 & 3 \\
	DBLP & Transductive & 17,716 & 105,734 & 1,639 & 4 \\
	\midrule
	Reddit & Inductive & 231,443 & 11,606,919 & 602 & 41 \\
	PPI & Inductive & \makecell{56,944 \\(24 graphs)} & 818,716 & 50 & \makecell{121 \\(multilabel)} \\
	\bottomrule
	\end{tabular}
\end{table}

\begin{table}[h]
	\small
	\centering
	\caption{Dataset download links.}
	\begin{tabular}{cl}
		\toprule
		Dataset & Download link \\
		\midrule
		Cora & \texttt{\url{https://github.com/kimiyoung/planetoid/raw/master/data}} \\
		Citeseer & \texttt{\url{https://github.com/kimiyoung/planetoid/raw/master/data}} \\
		Pubmed & \texttt{\url{https://github.com/kimiyoung/planetoid/raw/master/data}} \\
		DBLP & \texttt{\url{https://github.com/abojchevski/graph2gauss/raw/master/data/dblp.npz}} \\
		\midrule
		Reddit & \texttt{\url{https://s3.us-east-2.amazonaws.com/dgl.ai/dataset/reddit.zip}} \\
		PPI & \texttt{\url{https://s3.us-east-2.amazonaws.com/dgl.ai/dataset/ppi.zip}} \\
		\bottomrule
	\end{tabular}
	\label{tab:dataset-links}
\end{table}

\section{Implementation}
\label{appendix:implementation}

\paragraph{Computing infrastructures.}
All models are implemented using PyTorch Geometric 1.5.0 \cite{Fey:2019wv} and PyTorch 1.4.0 \cite{Paszke:2019vf}. All datasets used throughout experiments are available in PyTorch Geometric libraries. For node classification, we use the existing implementation of logistic regression with \(\ell_2\) regularization from Scikit-learn \cite{Pedregosa:2011cx}. All experiments are conducted on a computer server with eight NVIDIA Titan Xp GPUs (12GB memory each) and fourteen Intel Xeon E5-2660 v4 CPUs.

\paragraph{Hyperparameters.}
All models are initialized with Glorot initialization \cite{Glorot:2010uc}, and trained using Adam SGD optimizer \cite{Kingma:2015us} on all datasets.
The initial learning rate is set to 0.001 with an exception to 0.0005 on Cora and \(10^{-5}\) on Reddit.
The \(\ell_2\) weight decay factor is set to \(10^{-5}\) on all datasets.
On both transductive and inductive tasks, we train the model for a fixed number of epochs, specifically 200, 200, 1500, 1000 epochs for Cora, Citeseer, Pubmed and DBLP, respectively, 40 for Reddit and 200 for PPI.
The probability parameters controlling the sampling process, \(p_{r,1}, p_{m,1}\) for the first view and \(p_{r,2}, p_{m,2}\) for the second view, are all selected between 0.0 and 0.4, since the original graph will be overly corrupted when the probability is set too large. Note that to generate different contexts for nodes in the two views, \(p_{r,1}\) and \(p_{r,2}\) should be distinct, and the same holds for \(p_{m,1}\) and \(p_{m,2}\).
All dataset-specific hyperparameters are summarized in Table \ref{tab:hyperparameters}.

\begin{table}[h]
	\small
	\centering
	\caption{Hypeparameter specifications.}
    \begin{tabular}{cccccccccc}
	\toprule
	Dataset & \(p_{m,1}\) & \(p_{m,2}\) & \(p_{r,1}\) & \(p_{r,2}\) & \makecell{Learning\\rate} & \makecell{Weight\\decay} & \makecell{Training\\epochs} & \makecell{Hidden\\dimension} & \makecell{Activation\\function} \\
	\midrule
	Cora  & 0.3   & 0.4   & 0.2   & 0.4   & 0.005 & \(10^{-5}\) & 200   & 128   & ReLU \\
	Citeseer & 0.3   & 0.2   & 0.2   & 0.0   & 0.001 & \(10^{-5}\) & 200   & 256   & PReLU \\
	Pubmed & 0.0   & 0.2   & 0.4   & 0.1   & 0.001 & \(10^{-5}\) & 1,500  & 256   & ReLU \\
	DBLP  & 0.1   & 0.0   & 0.1   & 0.4   & 0.001 & \(10^{-5}\) & 1,000  & 256   & ReLU \\
	\midrule
	Reddit & 0.3   & 0.2   & 0.1   & 0.2   & 0.00001 & \(10^{-5}\) & 40    & 512   & ELU \\
	PPI   & 0.1   & 0.0   & 0.3   & 0.4   & 0.001 & \(10^{-5}\) & 500   & 512   & RReLU \\
	\bottomrule
	\end{tabular}
	\label{tab:hyperparameters}
\end{table}

%

\section{Additional Experiments}
\label{appendix:experiments}

\subsection{Sensitivity Analysis}
\label{appendix:sensitivity}

\begin{figure}[b]
	\centering
	\includegraphics[width=0.6\linewidth]{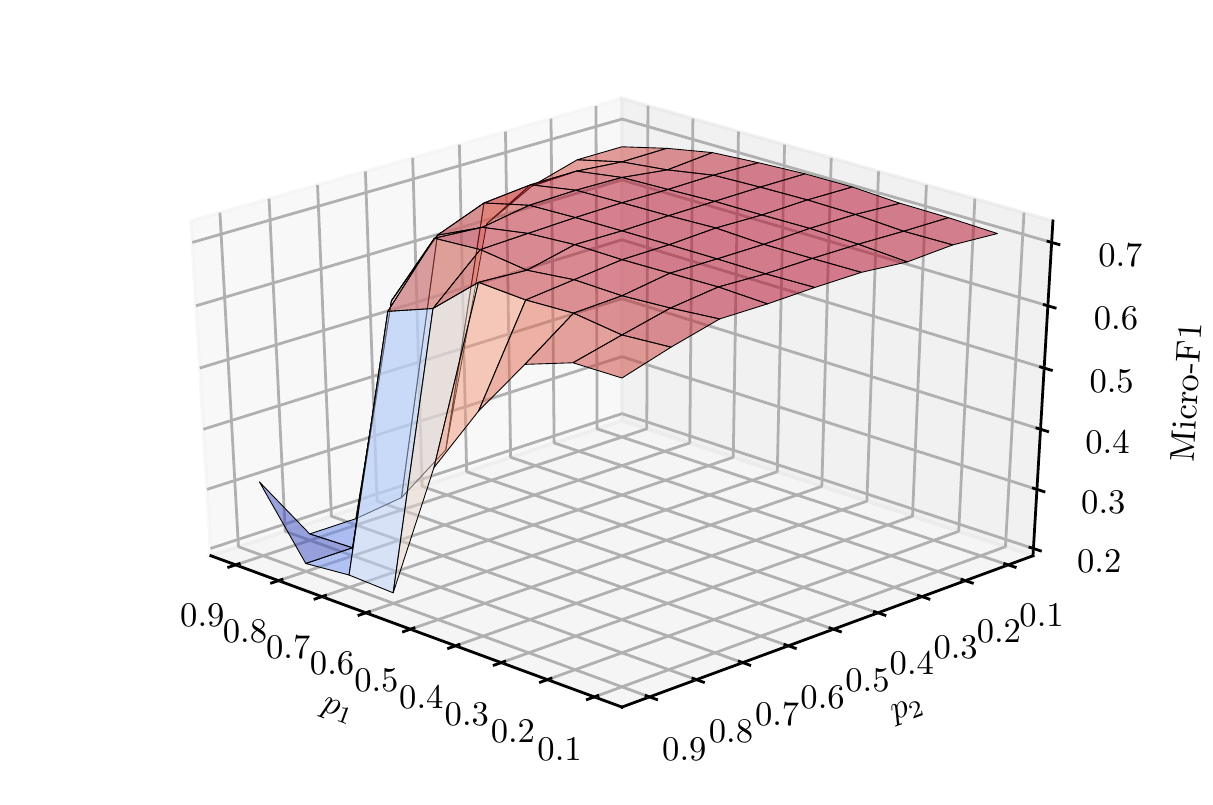}
	\caption{The performance of GRACE with varying different hyperparameters in transductive node classification on the Citeseer dataset in terms of Micro-F1.}
	\label{fig:sensitivity-analysis}
\end{figure}

In this section, we perform sensitivity analysis on critical hyperparameters in GRACE, namely four probabilities \(p_{m,1},p_{r,1},p_{m,2},p_{r,2}\) that determine the generation of graph views to show the model stability under the perturbation of these hyperparameters.
We conduct trasductive node classification by varying these parameters from 0.1 to 0.9. For sake of visualization brevity, we set \(p_1 = p_{r,1} = p_{m,1}\) and \(p_2 = p_{r,2} = p_{m,2}\). In other words, \(p_1\) and \(p_2\) control the generation of the two graph views. Note that we only change these four parameters in the sensitivity analysis, other parameters remain the same as previously described.

The results on the Citeseer dataset is shown are Figure \ref{fig:sensitivity-analysis}.
From the figure, it can be observed that the performance of node classification in terms of Micro-F1 is relatively stable when the parameters are not too large, as shown in the plateau in the figure. We thus conclude that overall, our model is insensitive to these probabilities, demonstrating the robustness to hyperparameter tuning.
If the probability is set too large (e.g., \(> 0.5\)), the original graph will be heavily undermined. For example, when \(p_r = 0.9\), almost every existing edge has been removed, leading to isolated nodes in the generated graph views. Then, under such circumstance, the graph convolutional network is hard to learn useful information from node neighborhoods. Therefore, the learnt node embeddings in the two graph views are not distinctive enough, which will result in difficulty of optimizing the contrastive objective.

\subsection{Ablation Studies}
\label{appendix:ablation}



In this section, we perform ablation studies on the two schemes for generating graph views, removing edge (RE) and masking node features (MF), to verify the effectiveness of the proposed hybrid scheme.
We denote GRACE (--RE) as the model without removing edges and GRACE (--MF) as the model without masking node features. We report the performance of GRACE (--RE), GRACE (--MF) and the original model GRACE on transductive node classification under the identical settings as previous, except for different enabled schemes. The results are presented in Table \ref{tab:ablation}.

It is seen that our hybrid approach that jointly applies RE and MF significantly outperform two downgraded models that only use one standalone method RE or MF. These results verify the effectiveness of our proposed scheme for graph corruption, and further show the necessity of jointly considering corruption at both graph topology and node feature levels.

\begin{table}[h]
\centering
\caption{The performance of model variants along with the original GRACE model in the ablation study in terms of accuracy of node classification. GRACE (--RE) and GRACE (--MF) denote the model without removing edges and masking node features respectively.}
\begin{tabular}{ccccc}
\toprule
Method & Cora & Citeseer & Pubmed & DBLP \\
\midrule
GRACE & \textbf{83.2{\footnotesize \textpm 0.5}} & \textbf{72.1{\footnotesize \textpm 0.5}} & \textbf{86.7{\footnotesize \textpm 0.1}} & \textbf{84.2{\footnotesize \textpm 0.1}} \\
GRACE (--RE) & 82.3{\footnotesize \textpm 0.4} & 72.0{\footnotesize \textpm 0.4} & 84.8{\footnotesize \textpm 0.2} & 83.6{\footnotesize \textpm 0.2} \\
GRACE (--MF) & 81.6{\footnotesize \textpm 0.4} & 69.9{\footnotesize \textpm 0.6} & 85.7{\footnotesize \textpm 0.1} & 83.5{\footnotesize \textpm 0.1} \\
\bottomrule
\end{tabular}
\label{tab:ablation}
\end{table}

\subsection{Comparison with InfoNCE Loss}
\label{appendix:InfoNCE-objective}

In this section, we consider another widely-used objective, the InfoNCE loss \cite{vandenOord:2018ut} , in contrastive methods. For fair comparison, we measure the node similarities between two graph views using the InfoNCE objective, which is defined as
\begin{equation}
	\mathcal{J}_{\text{NCE}} = \frac{1}{2} \left[ \ell_{\text{NCE}}(\bm V, \bm U) + \ell_{\text{NCE}}(\bm V, \bm U) \right],
\end{equation}
where the pairwise objective is defined by $\ell_{\text{NCE}}(\bm U, \bm V) \triangleq \frac{1}{N} \sum_{i=1}^N \log \frac{e^{\theta(\bm{u}_i, \bm{v}_i)}}{\frac{1}{N}\sum_{j = 1}^{N} e^{\theta(\bm{u}_i, \bm{v}_j)}}$. $\ell_{\text{NCE}}(\bm V, \bm U)$ can be defined symmetrically. The modified model is denoted as GRACE--NCE hereafter.
We report the performance of GRACE--NCE on transductive node classification under identical settings as with the original model GRACE. The results are summarized in Table \ref{tab:comparison-with-infonce}.

From the table, we can clearly see that the performance of the variant model GRACE-NCE is inferior to that of the original model GRACE on all four datasets. The results empirically demonstrate that, although InfoNCE is a stricter estimator of the mutual information, our objective is more effective and shows better downstream performance, which is consistent with previous observations in visual representation learning \cite{Tschannen:2020uo}. 
We believe that the superior performance of our objective compared to InfoNCE can be attributed to the inclusion of more negative samples. Specifically, we take intra-view negative pairs into consideration in our objective, which can be viewed as a regularization against the smoothing problem brought by graph convolution operators.

\begin{table}[h]
	\centering
	\caption{The performance of GRACE and GRACE--NCE in transductive node classification on four citation datasets.}
	\label{tab:comparison-with-infonce}	
	\begin{tabular}{ccccc}
	\toprule
	Method & Cora & Citeseer & Pubmed & DBLP \\
	\midrule
	GRACE & \textbf{83.2{\footnotesize \textpm0.5}} & \textbf{72.1{\footnotesize \textpm0.5}} & \textbf{86.7{\footnotesize \textpm0.1}} & \textbf{84.2{\footnotesize \textpm0.1}} \\
	GRACE--NCE & 82.1{\footnotesize \textpm0.4} & 70.9{\footnotesize \textpm0.6} & 85.0{\footnotesize \textpm0.1} & 82.1{\footnotesize \textpm0.1} \\
	\bottomrule
	\end{tabular}
\end{table} 

\subsection{Robustness to Sparse Features}
\label{appendix:robustness}

As discussed before, for existing work DGI, it is relatively easy to generate negative samples for nodes having dense features using the feature shuffling scheme. However, when node features are sparse, feature shuffling may not be sufficient to generate different neighborhoods for nodes, which motivates our hybrid scheme that corrupts the original graph at both topology and attribute levels. 

In this section, we conduct experiments with randomly contaminating the training data by masking a certain portion of the node features to zeros. Specifically, we vary the contamination rate of node features from 0.5 to 0.9 on four citation networks. We conduct experiments on transductive node classification with all other parameters being the same as previously described. The performance in terms of accuracy is plotted in Figure \ref{fig:robustness}.

From the figures, we can see that GRACE consistently outperforms DGI with large margins under different contamination rates, demonstrating the robustness of our proposed GRACE model to sparse features.
We attribute the robustness of GRACE to the superiority of our proposed RE method for graph corruption at topology level, since RE is capable of constructing different topology context for nodes without dependence on node features.
These results once again verify the necessity of considering graph corruption at both topology and attribute levels.
Note that, when a large portion of node features are masked, e.g., 90\% features are masked, both GRACE and DGI perform poorly. This may be explained from the fact that, when the node features are overly contaminated, nodes are highly sparse such that the GNN model is ineffective to extract useful information from nodes, leading to performance deterioration. 

\begin{figure}[h]
	\centering
	\subfloat[Cora]{
		\includegraphics[width=0.242\linewidth]{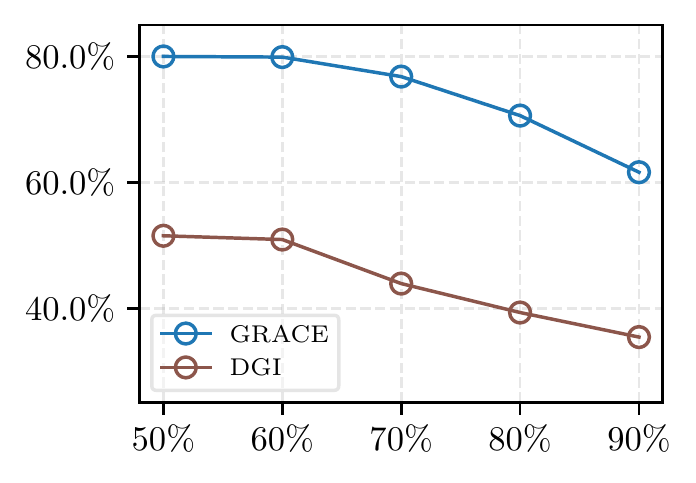}
		\label{fig:robustness-cora}
	}
	\subfloat[Citeseer]{
		\includegraphics[width=0.242\linewidth]{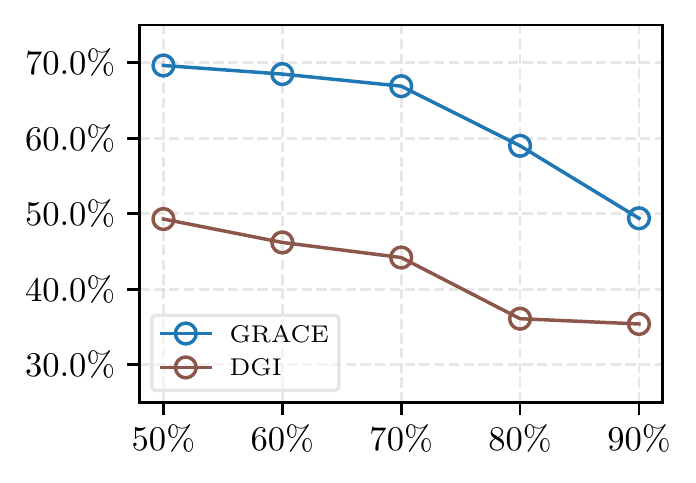}
		\label{fig:robustness-citeseer}
	}
	\subfloat[Pubmed]{
		\includegraphics[width=0.242\linewidth]{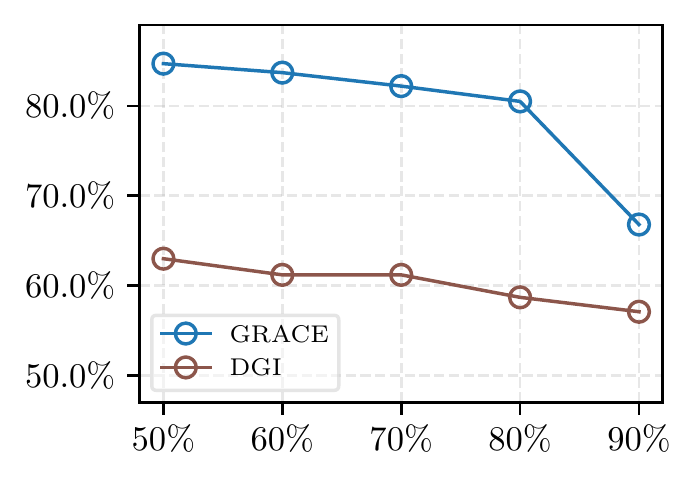}
		\label{fig:robustness-pubmed}
	}
	\subfloat[DBLP]{
		\includegraphics[width=0.242\linewidth]{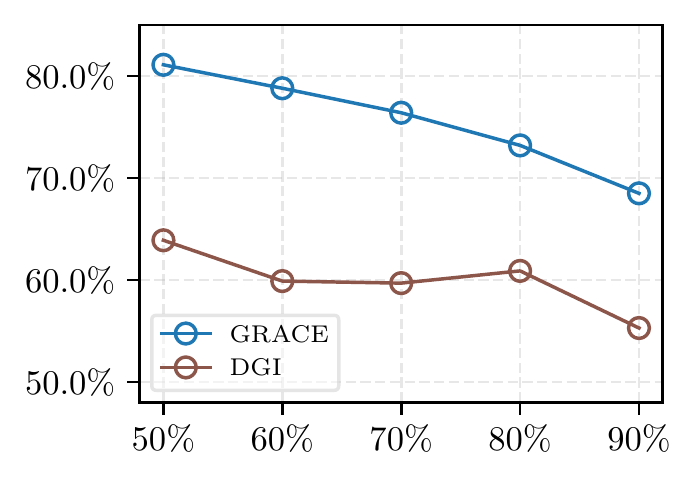}
		\label{fig:robustness-dblp}
	}
	\caption{The performance of GRACE and DGI in transductive node classification in terms of Micro-F1 on four citation datasets with a portion of node features masked under different masking rates.}
	\label{fig:robustness}
\end{figure}

\section{Detailed Proofs}
\label{appendix:proofs}

\subsection{Proof of Theorem 1}

\begin{theorem}
\label{thm:objective-InfoMax}
Let \(\mathbf{X}_i = \{ \bm{x}_k \}_{k \in \mathcal{N}(i)}\) be the neighborhood of node \(v_i\) that collectively maps to its output embedding, where \(\mathcal{N}(i)\) denotes the set of neighbors of node \(v_i\) specified by GNN architectures, and \(\mathbf{X}\) be the corresponding random variable with a uniform distribution \(p(\mathbf{X}_i) = \frac{1}{N}\). Given two random variables \(\mathbf{U, V} \in \mathbb{R}^{F^\prime}\) being the embedding in the two views, with their joint distribution denoted as \(p(\mathbf{U}, \mathbf{V})\), our objective \(\mathcal{J}\) is a lower bound of MI between encoder input \(\mathbf{X}\) and node representations in two graph views \(\mathbf{U, V}\). Formally,
\begin{equation}
	\mathcal{J} \leq I(\mathbf{X}; \mathbf{U}, \mathbf{V}).
\end{equation}
\end{theorem}

\begin{proof}
We first show the connection between our objective \(\mathcal{J}\) and the InfoNCE objective \cite{vandenOord:2018ut} , which can be defined as \cite{Poole:2019vk}
\[I_\text{NCE}(\mathbf U; \mathbf V) \triangleq \mathbb{E}_{\prod_{i} p( {\bm u}_i, {\bm v}_i)} \left[ \frac{1}{N} \sum_{i=1}^N \log \frac{e^{\theta(\bm{u}_i, \bm{v}_i)}}{\frac{1}{N}\sum_{j = 1}^{N} e^{\theta(\bm{u}_i, \bm{v}_j)}} \right], \]
where the critic function is defined as \(\theta (\bm{x}, \bm{y}) = s(g(\bm{x}), g(\bm{y}))\).
We further define \(\rho_r({\bm{u}}_i) = \sum_{j=1}^N \mathds 1_{[i \neq j]} \exp(\theta({\bm u}_i, {\bm u}_j) / \tau)\), \(\rho_c({\bm u}_i) = \sum_{j=1}^N \exp(\theta ({\bm u}_i, {\bm v}_j) / \tau)\) for convenience of notation.
Note that \(\rho_r({\bm v}_i)\) and \(\rho_c({\bm v}_i)\) can be defined symmetrically.
Then, our objective \(\mathcal{J}\) can be rewritten as
\begin{equation}
	\mathcal{J} = \mathbb E_{\prod_{i} p( {\bm u}_i, {\bm v}_i)} \left[ \frac{1}{N} \sum_{i=1}^N \log \frac {\exp(\theta( {\bm u}_i, {\bm v}_i) / \tau)} {\sqrt{\left( \rho_c( {\bm u}_i) + \rho_r({\bm u}_i) \right) \left( \rho_c( {\bm v}_i) + \rho_r({\bm v}_i) \right)}} \right].
\end{equation}
Using the notation of \(\rho_c\), the InfoNCE estimator \(I_\text{NCE}\) can be written as
\begin{equation}
	I_\text{NCE}( {\mathbf U}, {\mathbf V}) = \mathbb E_{\prod_{i} p( {\bm u}_i, {\bm v}_i)} \left[ \frac{1}{N} \sum_{i=1}^N \log \frac {\exp(\theta( {\bm u}_i, {\bm v}_i) / \tau)} {\rho_{c}( {\bm u}_i)} \right].
\end{equation}
Therefore,
\begin{equation}
	\begin{aligned}
		2\mathcal{J} & = I_\text{NCE}(\mathbf U, \mathbf V) - \mathbb E_{\prod_i p( {\mathbf u}_i,  {\bm v}_i)} \left[ \frac{1}{N} \sum_{i=1}^N \log \left( 1 + \frac {\rho_r( {\bm u}_i)} {\rho_c( {\bm u}_i)} \right) \right] \\
		& \quad + I_\text{NCE}(\mathbf V, \mathbf U) - \mathbb E_{\prod_i p( {\bm u}_i, {\bm v}_i)} \left[ \frac{1}{N} \sum_{i=1}^N \log \left( 1 + \frac {\rho_r( {\bm v}_i)} {\rho_c( {\bm v}_i)} \right) \right] \\
		& \leq I_\text{NCE}(\mathbf U, \mathbf V) + I_\text{NCE}(\mathbf V, \mathbf U). \\
	\end{aligned}
\end{equation}
According to \cite{Poole:2019vk}, the InfoNCE estimator is a lower bound of the true MI, i.e.
\begin{equation}
	I_\text{NCE}(\mathbf{U}, \mathbf{V}) \le I(\mathbf{U}; \mathbf{V}).
\end{equation}
Thus, we arrive at
\begin{equation}
	 2\mathcal{J} \leq I(\mathbf U; \mathbf V) + I(\mathbf V; \mathbf U) = 2 I(\mathbf U; \mathbf V),
\end{equation}
which leads to the inequality
\begin{equation}
\label{eq:objective-uv}
\mathcal J \le I( {\mathbf U}; {\mathbf V}).
\end{equation}

According to the data processing inequality, which states that, for all random variables $\mathbf{X}, \mathbf{Y}, \mathbf{Z}$ satisfying the Markov relation $\mathbf{X} \rightarrow \mathbf{Y} \rightarrow \mathbf{Z}$, the inequality $I(\mathbf{X}; \mathbf{Z}) \leq I(\mathbf{X}; \mathbf{Y})$ holds.
Then, we observe that $\mathbf{X}, \mathbf{U}, \mathbf{V}$ satisfy the relation $\mathbf{U} \leftarrow \mathbf{X} \rightarrow \mathbf{V}$. Since, $\mathbf{U}$ and $\mathbf{V}$ are conditionally independent after observing $\mathbf{X}$, the relation is Markov equivalent to $\mathbf{U} \rightarrow \mathbf{X} \rightarrow \mathbf{V}$, which leads to $I(\mathbf{U}; \mathbf{V}) \leq I(\mathbf{U}; \mathbf{X})$.
We further notice that the relation $\mathbf{X} \rightarrow (\mathbf{U}, \mathbf{V}) \rightarrow \mathbf{U}$ holds, and hence it follows that $I(\mathbf{X}; \mathbf{U}) \leq I(\mathbf{X}; \mathbf{U}, \mathbf{V})$.
Combining the two inequalities yields the required inequality
\begin{equation}
	\label{eq:data-processing}
	I(\mathbf U; \mathbf V) \leq I(\mathbf X; \mathbf U, \mathbf V).
\end{equation}
Following Eq. (\ref{eq:objective-uv}) and Eq. (\ref{eq:data-processing}), we finally arrive at inequality
\begin{equation}
	\mathcal{J} \leq I(\mathbf X; \mathbf U, \mathbf V),
\end{equation}
which concludes the proof.
\end{proof}

\subsection{Proof of Theorem 2}

\begin{theorem}
\label{thm:objective-triplet-loss}
When the projection function \(g\) is the identity function and we measure embedding similarity by simply taking inner product, i.e. \(s(\bm{u}, \bm{v}) = \bm{u}^\top \bm{v}\), and further assuming that positive pairs are far more aligned than negative pairs, minimizing the pairwise objective \(\ell(\bm u_i, \bm v_i)\) coincides with maximizing the triplet loss, as given in the sequel
\begin{equation}
	- \ell(\bm u_i, \bm v_i) \propto 4N \tau + \sum_{j=1}^N  \mathds 1_{[j \neq i]} \left[ \left(\| {\bm u_i} - {\bm v_i} \|^2 - \| {\bm u_i} - {\bm v_j} \|^2\right) + \left(\| {\bm u_i} - {\bm v_i} \|^2 - \| {\bm u_i} - {\bm u_j} \|^2\right) \right].
\end{equation}
\end{theorem}

\begin{proof}
Based on the assumptions, we can rearrange the pairwise objective as
\begin{equation}
	\begin{aligned}
		- \ell(\bm{u}_i, \bm{v}_i) & = - \log \frac {\exp \left( \bm{u}_i^\top \bm{v}_{i} / \tau\right)} {\sum_{k=1}^{N}  \exp \left( \bm{u}_i^\top \bm{v}_k / \tau\right) + \sum_{k=1}^{N} \mathds 1_{[k \neq i]} \exp \left( \bm{u}_i^\top \bm{u}_k / \tau\right)} \\
		& = \log \left( 1 + \sum_{k=1}^{N} \mathds 1_{[k \neq i]} \exp\left( \frac { {\bm u}_i^\top {\bm v}_k - {\bm u}_i^\top {\bm v}_i} {\tau} \right) + \sum_{k=1}^{N} \mathds 1_{[k \neq i]} \exp\left( \frac { {\bm u}_i^\top {\bm u}_k - {\bm u}_i^\top {\bm v}_i} {\tau} \right) \right).
	\end{aligned}
\end{equation}
By Taylor expansion of first order,
\begin{equation}
	\begin{aligned}
		- \ell( {\bm u}_i, {\bm v}_i) &\approx \sum_{k=1}^{N} \mathds 1_{[k \neq i]} \exp\left( \frac { {\bm u}_i^\top {\bm v}_k - {\bm u}_i^\top {\bm v}_i} {\tau} \right) + \sum_{k=1}^{N} \mathds 1_{[k \neq i]} \exp\left( \frac { {\bm u}_i^\top {\bm u}_k - {\bm u}_i^\top {\bm v}_i} {\tau} \right) \\
	& \approx 2 + \frac 1 \tau \left[ \sum_{k=1}^{N} \mathds 1_{[k \neq i]} (  {\bm u}_i^\top {\bm v}_k - {\bm u}_i^\top {\bm v}_i) + \sum_{k=1}^{N} \mathds 1_{[k \neq i]} ( {\bm u}_i^\top {\bm u}_k - {\bm u}_i^\top {\bm v}_i) \right] \\
	& = 2 - \frac 1 {2 \tau}\sum_{k=1}^{N} \mathds 1_{[k \neq i]}  \left[ \left(\| {\bm u_i} - {\bm v_k} \|^2 - \| {\bm u_i} - {\bm v_i} \|^2 \right) + \left(\| {\bm u_i} - {\bm u_k} \|^2 - \| {\bm u_i} - {\bm v_i} \|^2 \right) \right] \\
	& \propto 4N\tau + \sum_{k=1}^{N} \mathds 1_{[k \neq i]} \left[ \left(\| {\bm u_i} - {\bm v_i} \|^2 - \| {\bm u_i} - {\bm v_k} \|^2\right) + \left(\| {\bm u_i} - {\bm v_i} \|^2 - \| {\bm u_i} - {\bm u_k} \|^2\right) \right],
	\end{aligned}
\end{equation}
which concludes the proof.
\end{proof}

\small
\bibliographystyle{unsrtnat}
\bibliography{neurips_2020}

\end{document}